\documentclass[11pt]{article}

\usepackage{subcaption}
\usepackage[utf8]{inputenc} 
\usepackage[T1]{fontenc}    
\usepackage{hyperref}       
\usepackage{url}            
\usepackage{booktabs}       
\usepackage{amsfonts}       
\usepackage{nicefrac}       
\usepackage{microtype}      

\usepackage{natbib}
\usepackage[hmarginratio=1:1, vmarginratio =1:1, textheight=22cm, textwidth=18cm]{geometry}
\usepackage{amsthm,amsmath,amssymb}
\usepackage{algorithm}
\usepackage[noend]{algpseudocode}
\usepackage{graphicx,xcolor}

\makeatletter
\def\BState{\State\hskip-\ALG@thistlm}
\makeatother

\newtheorem{theorem}{Theorem}[section]

\newtheorem{lemma}[theorem]{Lemma}

\newtheorem{remark}{Remark}[section]

\newcommand{\G}{{\mathcal{G}}}

\newcounter{rcnt}[section]

\makeatletter
\newcommand{\vast}{\bBigg@{3}}
\newcommand{\Vast}{\bBigg@{4}}
\makeatother

\def\G{\mathcal{G}}
\def\T{\mathcal{T}}

\def\detwo{\mathsf{DE_2}}
\def\nbs{\mathsf{NBS}}
\def\usvt{\mathsf{USVT}}
\def\de{\mathsf{DE}}
\def\nbse{\mathsf{NBSE}}
\def\dist{\mathsf{DIST}}
\def\ave{\mathrm{ave}}

\def\corr{\mathrm{corr}}

\def\ber#1{\mathrm{Ber}({#1})}

\title{\textbf{Graphon Estimation from Partially Observed Network Data}}

\author{%
  \large{Soumendu Sundar Mukherjee} \\ 
  \large{Interdisciplinary Statistical Research Unit (ISRU)}\\
  \large{Indian Statistical Institute, Kolkata}\\
  \large{Kolkata 700108, India}\\
 \large{\texttt{soumendu041@gmail.com}} \\
 \and
  \large{Sayak Chakrabarti}\\
  \large{Computer Science and Engineering}\\
  \large{Indian Institute of Technology, Kanpur}\\
  \large{Kanpur, UP 208016, India}\\
 \large{\texttt{sayak@iitk.ac.in}}
}

\usepackage{fancyhdr}
\begin{document}

\maketitle

\begin{abstract}
 We consider estimating the edge-probability matrix of a network generated from
a graphon model when the full network is not observed---only some overlapping
subgraphs are. We extend the neighbourhood smoothing ($\nbs$) algorithm of
\citet{zhang2017estimating} to this missing-data set-up and show experimentally that, for a wide range of
graphons, the extended $\nbs$ algorithm achieves significantly smaller error rates
than standard graphon estimation algorithms such as vanilla neighbourhood
smoothing ($\nbs$), universal singular value thresholding ($\usvt$), blockmodel
approximation, matrix completion, etc. We also show that the extended $\nbs$ algorithm is much more robust to missing data.
\end{abstract}

\section{Introduction}\label{sec:intro}
Graphons are limits of dense graph-sequences \citet{lovasz2012large}. A graphon $h$ is a symmetric measurable function from $[0, 1]^2$ to $[0, 1]$. Probabilistic models of networks based on graphons are a particular type of latent space model. The first such model is due to \citet{bollobas2007phase}. \citet{bickel2009nonparametric} also considered this model. A graphon is used here as a non-parametric link function---there are i.i.d. $\mathrm{Uniform}(0, 1)$ latent characteristics $\xi_i$ of each node in a network, and given these characteristics, an edge is formed between a pair of nodes $i$ and $j$, independently of all other edges, with probability $P_{ij} = h(\xi_i, \xi_j)$. Now, there are inherent unidentifiability issues in such a model. For any measure preserving bijection $\phi$ on $[0, 1]^2$, the graphon $h \circ \phi$ also gives rise to the same probability model. Thus graphons can only be estimated up to equivalences classes. Or to remove the identifiability issue, one needs to make further assumptions on the graphon, e.g. monotone degrees \citet{chan2014consistent}. The other estimation problem is that of estimating the probability matrix $P$ coming from a graphon, which is a well-defined problem. In this paper, we will mean by graphon estimation this latter problem.

The problem of estimating the underlying graphon from an observed network has attracted much attention in recent past. \citet{airoldi2013stochastic} considered a stochastic blockmodel approximation to a graphon. \citet{zhang2017estimating} devised an elegant neighbourhood smoothing estimator for graphons. The universal singular value thresholding ($\usvt$) method of \citet{chatterjee2015matrix} is capable of estimating graphons. The matrix completion method of \citet{keshavan2010matrix} can also be used for this purpose. \citet{gao2015rate} obtained the minimax rate for graphon estimation and proposed a combinatorial algorithm that achieves it.

Although graphon estimation has been studied in some detail for fully observed networks, its study under missing data set-ups is perhaps more important. This is because when collecting network data one is hardly certain about all the edges. There is enormous scope of application if one is able to predict links when one suspects that a zero in the adjacency matrix is possibly indicating missing data and not the absence of an edge. More generally, such link prediction problems have been considered by many authors (\citet{liben2007link,al2006link,lu2011link}).

Recently, graphon estimation under a missing data set-up where one observes full ego-networks of some (but not all) individuals in a network has been carried out in \citet{wu2018link}. In general, the problem of link prediction with partially observed data has been tackled before in \citet{zhao2017link, gaucher2019maximum}.

In this paper too, we study graphon estimation under a missing data set-up. The missing data model we study is very different from that of \citet{wu2018link}. Instead of ego-networks as in their paper, one observes, in our model, certain overlapping subgraphs. We extend the neighbourhood smoothing estimator of \citet{zhang2017estimating} to this missing data set-up by devising a method based on the triangle inequality to extend a distance matrix to all of the network, when one actually has some estimate of the distances within the overlapping subgraphs. The case where there are only two overlapping subgraphs is easier to tackle and we use this as a building block for a more general algorithm for the case when there are more than two overlapping subgraphs.

Through extensive numerical study on simulated and real world graphs, we show that the extended $\nbs$ algorithm, for a wide range of graphons, vastly outperforms standard graphon estimation methods such as vanilla neighbourhood
smoothing ($\nbs$), universal singular value thresholding ($\usvt$), blockmodel approximation, matrix completion, etc.

The rest of the paper is organized as follows. In Section~\ref{sec:setup} we describe in detail the missing data model we consider. Then, in Section~\ref{sec:methodology}, we briefly recap the $\nbs$ algorithm of \citet{zhang2017estimating} and then extend it to our set-up. Section~\ref{sec:exp} contains our empirical results. We finally end with some concluding remarks and future directions in Section~\ref{sec:conc}.

\section{Problem set-up}\label{sec:setup}
Suppose we observe $T$ subgraphs $G_1 = (V_1, E_1), \ldots, G_T = (V_T, E_T)$ of some simple undirected graph $G = (V, E)$ on $|V| = n$ vertices, generated from some graphon $h$. For ease of notation we will take $V = \{1, \ldots, n\}$. Furthermore, assume that the vertex sets of these subgraphs have some intersection. To make it precise, define a super-graph $\G$ on $T$ nodes, where the $i$-th node represents $V_i$. Put an edge between nodes $i$ and $j$ if $V_i \cap V_j \ne \emptyset$. Assume that $\G$ is connected. We also assume that $V = \cup_t V_t$, i.e. these subgraphs cover the whole graph. 

Thus, if $A$ denotes the adjacency matrix of $G$, then there are (unobserved) $\mathrm{Uniform}(0, 1)$ variables $\xi_i, 1 \le i \le n$, such that 
\begin{equation}
  A_{k\ell} \overset{i.i.d.}{\sim} \ber{h(\xi_k, \xi_\ell)},  1\le k\ne \ell \le n.
\end{equation}
By $A_t$ we denote the adjacency matrix of network $G_t$. Then, using a subsetting notation, $A_t = A(V_t, V_t)$. Let 
\begin{equation}
  \mathcal{O} := \cup_{\ell = 1}^T V_\ell \times V_\ell = \{(i, j) \mid i, j \in V_\ell \text{ for some } \ell = 1, \ldots, T\},
\end{equation}
be the set of observed pairs. Now we observe the $n \times n$ matrix $A^{\mathrm{obs}}$, where  
\begin{equation}
  A^{\mathrm{obs}}_{k\ell} = \begin{cases} 

  A_{k\ell} & \text{if } (k, \ell) \in \mathcal{O}, \text{ and} \\
  0 & \text{otherwise}.
  \end{cases} 
\end{equation}

The goal is to estimate the probability matrix $P$, where $P_{k\ell} = h(\xi_k, \xi_\ell)$, given $A^{\mathrm{obs}}$. See Figure~\ref{fig:overlap_schematic} for an example of this set-up.

\begin{remark}
Although we are assuming that there is some big network $G$ of which some overlapping subgraphs $G_t$ are observed, it is quite straightforward to adapt our approach to the case where one observes $T$ graphs coming from the same graphon $h$ where the vertex sets of these graphs have some intersection in the sense of the corresponding super-graph being connected. 
\end{remark}

\begin{figure}[!htbp]
\caption{(a) A graph on 16 nodes, with 4 observed overlapping subgraphs; (b) the corresponding super-graph $\G$.}
\label{fig:overlap_schematic}
\centering
\begin{tabular}{cc}
\includegraphics[width = 0.35\textwidth]{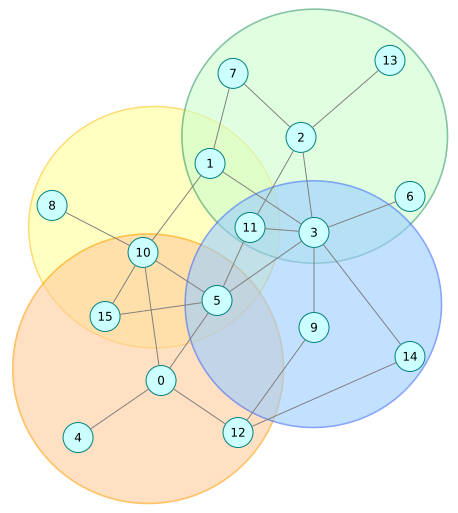} & \includegraphics[width = 0.6\textwidth]{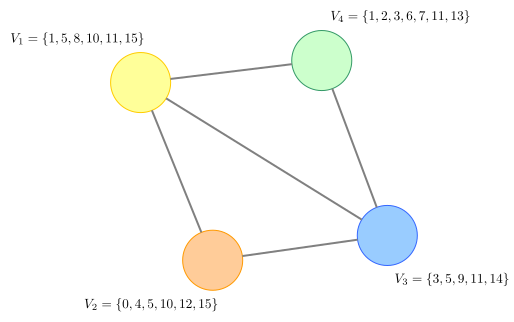} \\
(a) & (b)
\end{tabular}
\end{figure}

\section{Methodology}\label{sec:methodology}
We will generalize the neighbourhood smoothing approach of \citet{zhang2017estimating}. Their approach is to construct a certain neighbourhood $N_i$ for each node $i$. This is done by first calculating a distance measure $\bar{d}(i, j)$ between each pair of vertices $(i, j)$, and then saying that $j \in N_i$ if $\bar{d}(i, j)$ is less than a certain threshold. Once these neighbourhoods have been constructed, $P_{ij}$ can estimated by
\begin{equation}
  \tilde{P}_{ij} = \frac{\sum_{i' \in N_i}A_{i'j}}{|N_i|}.
\end{equation}
However, this is not symmetric, so \citet{zhang2017estimating} take $(\tilde{P}_{ij} + \tilde{P}_{ji})/2$ as the final estimate.

The distance measure that \citet{zhang2017estimating} use is
\begin{equation}
\bar{d}(i, i') = \sqrt{\max_{k \neq i, i'} \frac{\langle A_i - A_{i'}, A_k \rangle}{n}}.
\end{equation}
We refer the reader to Section 2.2 of \citet{zhang2017estimating} for details on how one obtains this distance measure. One thing to note here is that $\bar{d}$ is not a metric, it tries to approximate one though.

By $\dist(A)$ we denote an algorithm that takes as input the adjacency matrix $A$, and outputs a distance matrix $D$ where $D_{ij} = \bar{d}(i, j)$. Once we have such a distance matrix $D$, the next algorithm does neighbourhood smoothing.


\begin{algorithm}[H]
  \caption{$\nbs$: Neighbourhood Smoothing by \citet{zhang2017estimating}. Input: the graph adjacency matrix $A$, and a distance matrix $D$.}
  \label{alg_nbs}
  \begin{algorithmic}[1]

    \State \textbf{Neighbourhood construction:} Let $q_i(h)$ denote the $h$-th sample quantile of the set $\{D_{ii'} \mid i \neq i'\}$, where $h$ is a tuning parameter. Set, for $i = 1, \ldots, n$,
    \[
      N_i = \{i' \mid i' \neq i, D_{ii'} \leq  q_i(h)\}.
    \]

    \State \textbf{Neighbourhood smoothing:} For all $i \ne j$, set
    \[
      \hat{P}_{ij} = \frac{1}{2}\bigg(\frac{\sum_{i' \in N_i}A_{i'j}}{|N_i|} + \frac{\sum_{i' \in N_j}A_{i'i}}{|N_j|}\bigg).
    \]  

    \State Return $\hat{P}$.
  \end{algorithmic}
\end{algorithm}
From their theoretical considerations, \citet{zhang2017estimating} recommend the choice $h \asymp \sqrt{\frac{\log n}{n}}$. We will also use this recommendation in our extended $\nbs$ algorithm.
\subsection{Distance Extension}
We first discuss the $T = 2$ case, which will then be used to tackle the general case.
\subsubsection{The \texorpdfstring{$T = 2$}{} case}

Suppose, like \citet{zhang2017estimating}, we have a measure of distance $d(i, i')$ between the nodes of a network. We can use the triangle inequality and the common intersection between the two networks to estimate distances between nodes that are part of different graphs. To elaborate, if $i \in V_1\setminus V_2$, $j \in V_2\setminus V_1$, then we define
\begin{equation}
  d_M(i, j) = \min_{r \in V_1 \cap V_2} d(i, r) + d(j, r).
\end{equation}
The triangle inequality can be used again to obtain a lower bound. Since $d(i, j) \ge |d(i, r) - d(j, r)|$, we define
\begin{equation}
  d_m(i,j) = \max_{r \in V_1 \cap V_2} |d(i, r) - d(j, r)|.
\end{equation}
If $d$ were a true distance, then we would have
\begin{equation}
  d_m(i, j) \le d(i, j) \le d_M(i, j).
\end{equation}
So we may take our estimate to be some average of $d_m$ and $d_M$.
\begin{equation}
  \tilde{d}(i, j):= \ave(d_m(i, j), d_M(i, j)).
\end{equation}
Experimentally we did not find much differences between different types of averages. Overall, the harmonic mean $\ave(x, y) = \frac{2 x y}{x + y}$ seemed to perform well.

Also, on $V_1 \cap V_2$ we may have two potentially different values of $\tilde{d}(i, j)$ coming from the two different graphs. We choose the arithmetic mean of these two values and assign that to $\tilde{d}(i, j)$.

Thus we have a measure of distance between any two vertices in $V$. So we can define a neighbourhood smoothing estimator of $P$ just like \citet{zhang2017estimating}. To that end, we first describe the distance extension algorithm.
\begin{algorithm}[H]
  \caption{$\detwo$: Distance Extension for $T = 2$. Input: two overlapping subgraphs $G_1$, $G_2$.}
  \label{alg_de2}
  \begin{algorithmic}[1]
    \State \textbf{Distance calculation for subgraphs:} Use the $\dist$ algorithm to calculate $D_1$ and $D_2$, node-distance matrices for $G_1$ and $G_2$.

    \State \textbf{Extending the distance:} For $i \in V_1 \setminus V_2$, $j \in V_2 \setminus V_1$, set
    \[
    D_{ij} = \ave(\min_{r \in V_1 \cap V_2}(D_{1,ir} + D_{2,jr}), \max_{r \in V_1 \cap V_2}|D_{1,ir} - D_{2,jr}|).
    \]
    \State Return $(G_1 \cup G_2, D)$.
  \end{algorithmic}
\end{algorithm}
\subsubsection{The general case}
 In this case, we have $T$ overlapping subgraphs. As described in Section~\ref{sec:setup}, it is more illuminating to consider the super-graph $\G$ on $T$ nodes, where the $i$-th node $V_i$ and there is an edge between nodes $i$ and $j$ if $V_i \cap V_j \ne \emptyset$. We assume that $\G$ is connected. Given $i \in V_a$, $i' \in V_b$, we will try to estimate $d(i, i')$ using the overlaps. As $\G$ is connected, there is a path of overlapping subgraphs $V_a \sim V_{\ell_1} \sim \cdots \sim V_{\ell_t} \sim V_b$. 

 The issue is that, e.g., computing max of sum of distances along all possible chains of between vertices from these overlapping graphs is expensive (for $T = 2$, this was fine). So, as a compromise, we take a spanning tree $\T$ of $\G$. On this tree, we visit each node on a particular traversal $\tau = (\tau_r)_{r = 1}^k$, a finite sequence of adjacent nodes of $\T$ which covers all the vertices. Say that the traversal is $G_{\tau_1} \rightarrow G_{\tau_2} \rightarrow \cdots \rightarrow G_{\tau_\ell}$. At point $k + 1$ of the traversal, we apply $\detwo$ on $\cup_{r = 1}^k G_{\tau_r}$ and $G_{\tau_{k + 1}}$ to get a distance matrix on all of  $\cup_{r = 1}^{k + 1} G_{i_r}$. At the end we get a distance matrix $D_{\T, \tau}$ that depends on both the tree $\T$ and the particular traversal $\tau$. Finally, we do this several times over a number of spanning trees $\T$ and traversals $\tau$ thereof, and take the average of all the resulting $D_{\T, \tau}$ as the final estimate of $D$. 

We now describe this algorithm in detail.
\begin{algorithm}[H]
    \caption{$\de$: Distance Extension. Input: a spanning tree $\T$ of the super-graph $\G$, and a traversal $\tau = (\tau_r)_{r = 1}^\ell$ of $\T$.}\label{alg_de}
    \begin{algorithmic}[1]
      \State Set $(G_{full}, D_{full}) = (G_{\tau_1}, D_{\tau_1})$.
      \For{ $k$ in $2, \ldots, \ell$}
      \If{$\tau_k \notin \{\tau_1, \ldots, \tau_{k - 1}\}$}
      \State Update: $(G_{full}, D_{full}) \leftarrow \detwo((G_{full}, D_{full}), (G_{\tau_{k}}, D_{\tau_k}))$.
      \EndIf
      \EndFor
      \State Return $(G_{full}, D_{full})$.
    \end{algorithmic}
\end{algorithm}

   \textbf{What spanning trees to take?} We can take several uniformly random spanning trees. An alternative to this is to take a maximal spanning tree $\T$ of $\G$ where the edge-weights are the overlaps, and consider its traversals only. In Section~\ref{sec:exp} we perform experiments to show the impact of traversals.

      
      

\subsection{Neighbourhood Smoothing, Extended}
Once we have computed a distance matrix on the full graph, we can do the usual neighbourhood smoothing on the matrix $A$. Now we describe this extended $\nbs$ algorithm.

\begin{algorithm}[H]
  \caption{$\nbse_0$: Neighbourhood Smoothing, Extended (baby version). Input: $G_1, \ldots, G_T$.}
  \label{mainalg_baby}
  \begin{algorithmic}[1]
    \State \textbf{Distance calculation:}
    Take $I$ trees $\T_1, \ldots, \T_I$ and, for each tree $\T_i$, $J$ traversals $\tau_{1}^{(i)}, \ldots, \tau_{J}^{(i)}$. Construct the distance matrix estimate
  \[
    D = \frac{1}{IJ}\sum_{i = 1}^{I}\sum_{j = 1}^J \de(\T_i, \tau_{j}^{(i)}).
  \]
    \vskip5pt
    \State \textbf{Neighbourhood smoothing:} $\hat{P} = \nbs(A^{\mathrm{obs}}, D).$
  \end{algorithmic}
\end{algorithm}

Note that so far our goal has been to estimate the neighbourhoods better than what a vanilla $\nbs$ algorithm would do.
However, when we estimate $P_{ij}$ as done in $\nbse_0$, we are underestimating the numerator, because we are replacing unobserved edges by $0$. This can be corrected for to some extent by the following prescription: Let $\hat{P}^{(0)}$ be the estimate we get from $\nbse_0$, and let $N_i$ denotes the neighbourhood of $i$ constructed in $\nbse_0$. We then correct $\hat{P}^{(0)}_{ij}$ by replacing unobserved edges by their corresponding estimated edge probabilities, obtained from $\nbse_0$:

\begin{equation}
  \hat{P}_{ij}^{(1)} = \frac{1}{2}\Vast(\frac{\sum_{\stackrel{i' \in N_i}{(i',j) \in \mathcal{O}}}A_{i'j} + \sum_{\stackrel{i' \in N_i}{(i',j) \notin \mathcal{O}}}\hat{P}^{(0)}_{i'j}}{|N_i|} + \frac{\sum_{\stackrel{i' \in N_j}{(i',i) \in \mathcal{O}}}A_{i'i} + \sum_{\stackrel{i' \in N_j}{(i',i) \notin \mathcal{O}}}\hat{P}^{(0)}_{i'i}}{|N_j|}\Vast).
\end{equation}
Let us denote the above procedure as $F_{\corr}$, i.e. $\hat{P}^{(1)} = F_{\corr}(\hat{P}^{(0)})$. This procedure can be repeated a few times until the estimates get stable. That is, after we get $\hat{P}^{(1)}$, we can correct it further by the same procedure: $\hat{P}^{(2)} = F_{\corr}(\hat{P}^{(1)})$ and so on. The following simple lemma shows that this iterative scheme always converges.

\begin{lemma}
The iterations
\[
  \hat{P}^{(t + 1)} = F_{corr}(\hat{P}^{(t)}), \hat{P}^{(0)} = \nbse(G_1, \ldots, G_T)
\]
increase to a probability matrix $\hat{P}^{(\infty)}$.
\end{lemma}
\begin{proof}
Note first that $\hat{P}^{(1)}_{ij} \ge \hat{P}^{(0)}_{ij}$ for all $i, j$, because, for an unobserved pair $(k, k')$, we have $A_{kk'} = 0 \le \hat{P}^{(0)}_{kl}$. Also,
\begin{equation}\label{eq:difference}
  \hat{P}^{(t + 2)}_{ij} - \hat{P}^{(t + 1)}_{ij} = \frac{1}{2}\Vast( \frac{\sum_{\stackrel{i' \in N_i}{(i',j) \notin \mathcal{O}}}\hat{P}^{(t+1)}_{i'j} - \hat{P}^{(t)}_{i'j}}{|N_i|} + \frac{\sum_{\stackrel{i' \in N_j}{(i',i) \notin \mathcal{O}}}\hat{P}^{(t + 1)}_{i'i} - \hat{P}^{(t)}_{i'i}}{|N_j|} \Vast).
\end{equation}
Hence, $P^{(2)}_{ij} \ge \hat{P}^{(1)}_{ij}$, and so on. That is, $P^{(t)}_{ij}$ is an increasing sequence. But, we have the trivial upper bound
\[
  \hat{P}_{ij}^{(t)} \le \frac{1}{2}\Vast(\frac{\sum_{\stackrel{i' \in N_i}{(i',j) \in \mathcal{O}}}A_{i'j} + \sum_{\stackrel{i' \in N_i}{(i',j) \notin \mathcal{O}}}1}{|N_i|} + \frac{\sum_{\stackrel{i' \in N_j}{(i',i) \in \mathcal{O}}}A_{i'i} + \sum_{\stackrel{i' \in N_j}{(i',i) \notin \mathcal{O}}}1}{|N_j|}\Vast).
\]
Therefore, being an increasing sequence bounded from above, $\hat{P}^{(t)}$ converges to some $\hat{P}^{(\infty)}_{ij} \in [0, 1]$.
\end{proof}

In practice, we continue these iterations until $n^{-1}\|\hat{P}^{(t + 1)} - \hat{P}^{(t)}\|_F$ becomes smaller than a pre-specified threshold. Now we are in a position to describe the full algorithm.

\begin{algorithm}[H]
  \caption{$\nbse$: Neighbourhood Smoothing, Extended. Input: $G_1, \ldots, G_T$, a threshold $\epsilon$.}
  \label{mainalg}
  \begin{algorithmic}[1]
    \State \textbf{Initialization:}
    Set $\hat{P}^{(0)} = \nbse_0(G_1, \ldots, G_T), \Delta = \epsilon, t = 0$.

    \vskip5pt
    \While{$\Delta \ge \epsilon$}
      \State $t = t + 1$.
      \State $P^{(t)} = F_{\corr}(P^{(t - 1)})$.
      \State $\Delta = n^{-1}\|P^{(t)} - P^{(t - 1)}\|_F$.
    \EndWhile
    \State Return $P^{(t)}$.
  \end{algorithmic}
\end{algorithm}




\section{Results}\label{sec:exp}
\subsection{Simulations}
In the $T = 2$ case, We generated networks of size $n=1000$ from six graphons (see Figure~\ref{fig:graphon-examples}), observed were two random subgraphs of size $500+m$, where $m$ controls the size of the overlap. See Figure~\ref{fig:graphon-t2-err} for a comparison between various algorithms. In all of these, we see huge improvement achieved by $\nbse$ especially when $m$ is small.
\begin{figure}[!htbp]
  \caption{Heatmaps of some graphons: (a) $h(x, y) = \sin(5 \pi (x + y - 1) + 1)/2 + 0.5$; (b) $h(x, y) = 1 - 0.5 \max(x, y)$; (c) $h(x, y) = 1 - (1 + \exp(-15(0.8 |x - y|))^{4/5} - 0.1))^{-1}$; (d) $h(x, y) = \frac{x^2 + y^2}{3} \cos\big(\frac{1}{x^2 + y^2} \big) + 0.15$; (e) $h(x, y) = \frac{1}{1 + \exp(-x -y)}$; (f) $h(x, y) = 0.3 \,\mathbf{I}_{\{\lfloor 2x \rfloor = \lfloor 2y \rfloor\}} + 0.03 \,\mathbf{I}_{\{\lfloor 2x \rfloor \ne \lfloor 2y \rfloor\}}$.}
     \label{fig:graphon-examples}
   \centering
   \begin{tabular}{ccc}
   \includegraphics[width=0.30\textwidth]{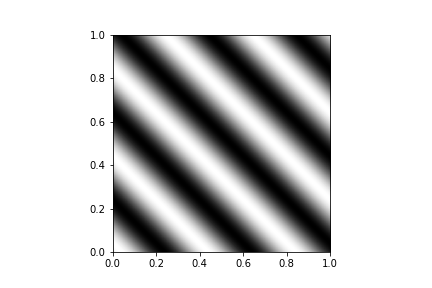} & \includegraphics[width=0.30\textwidth]{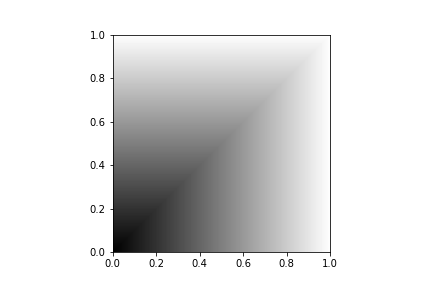} & \includegraphics[width=0.30\textwidth]{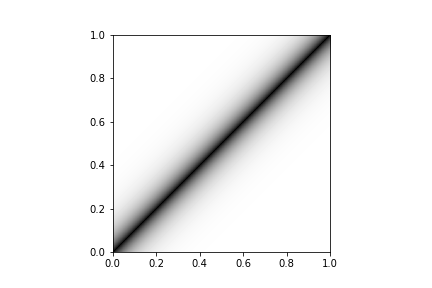} \\
    (a) & (b) & (c) \\
  \includegraphics[width=0.30\textwidth]{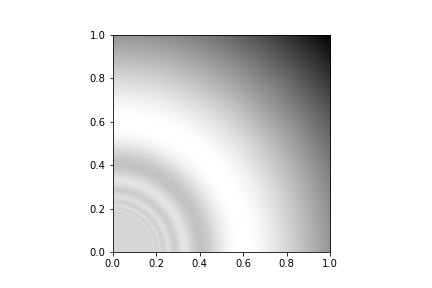} & \includegraphics[width=0.30\textwidth]{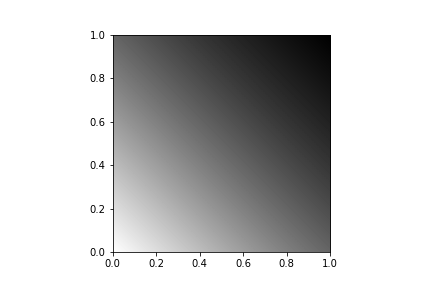} & \includegraphics[width=0.30\textwidth]{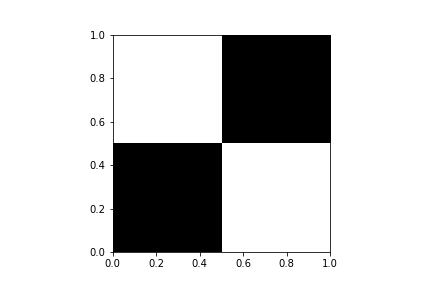} \\
     (d) & (e) & (f) \\
   \end{tabular}
\end{figure}

In another experiment, we consider a missing data scenario as depicted in Figure~\ref{fig:covers}-(b). Five subgraphs were observed. In Figure~\ref{fig_barplot}, we plot the estimation error of $\nbse$ for different traversals (traversals are paths in this example) and also for the case when we take an averaged distance matrix over paths as described in the $\nbse_0$ algorithm. In general, we do not find any significant impact of traversals on the estimation error. We have done this experiment in other missing data scenarios as well and have arrived at the same conclusion. In Table~\ref{table1}, we compare $\nbse$ (with maximal spanning path) against other algorithms.

\begin{remark}
Because of our missing data model, a significant number of elements close to the diagonal of the adjacency matrix are observed. Therefore, probability matrix estimation suffers the least from our missing data model for graphons that are strongly concentrated near the $x = y$ line. This effect is clearly seen in graphon (c) and also to some extent in graphon (f) (see Figure~\ref{fig:graphon-t2-err} and Table~\ref{table1}).
\end{remark}

\begin{figure}[!htbp]
  \caption{Comparison between various algorithms ($T = 2$): NBS = Neighbourhood Smoothing; SBA = Stochastic Blockmodel Approximation; EDS = Empirical Degree Sorting; MC = Matrix Completion; USVT = Universal Singular Value Thresholding; NBSE = Neighbourhood Smoothing, Extended. We are plotting the estimation error (averaged over $5$ independent replications) as a function of the overlap $m$ in the scale of $n$.}
   \label{fig:graphon-t2-err}
   \centering
   \begin{tabular}{ccc}
   \includegraphics[width = 0.3\textwidth]{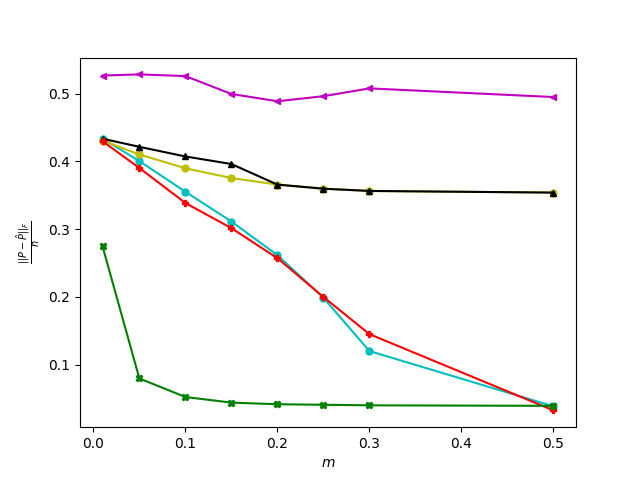} & \includegraphics[width = 0.3\textwidth]{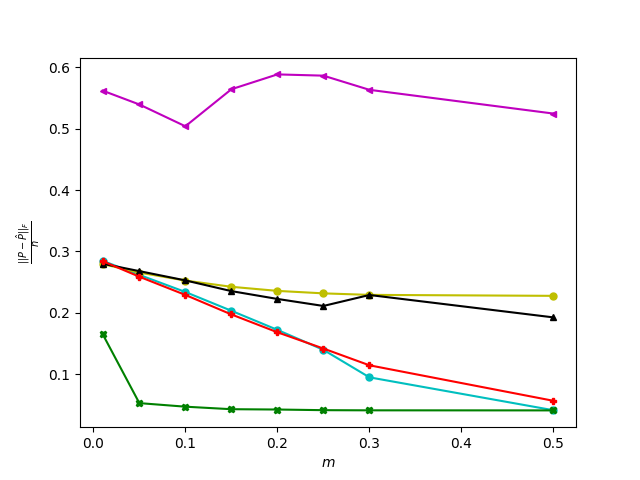} & \includegraphics[width = 0.3\textwidth]{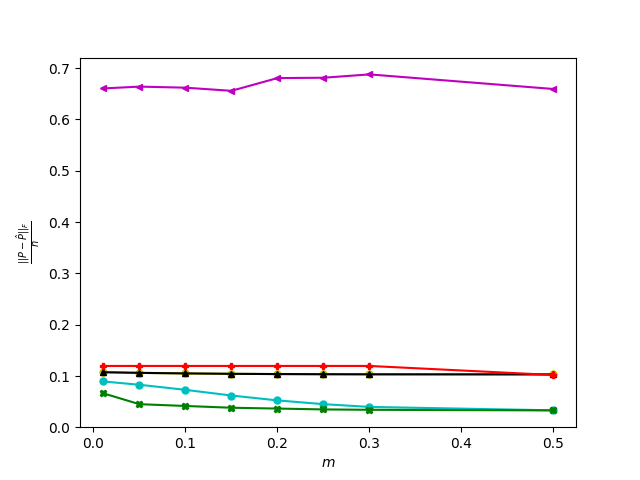} \\
   (a) & (b) & (c) \\
   \includegraphics[width = 0.3\textwidth]{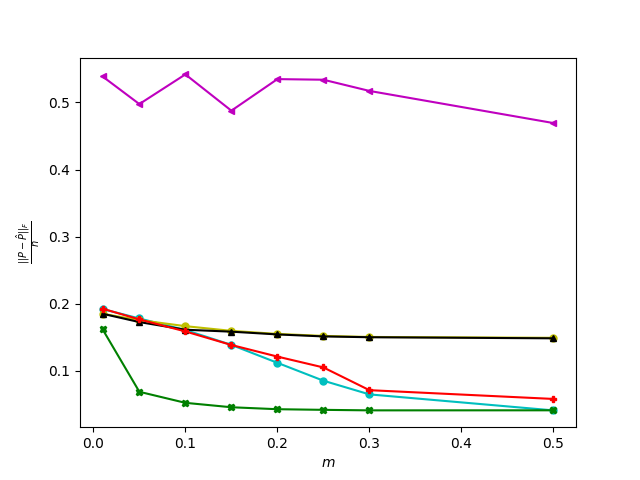} & \includegraphics[width = 0.3\textwidth]{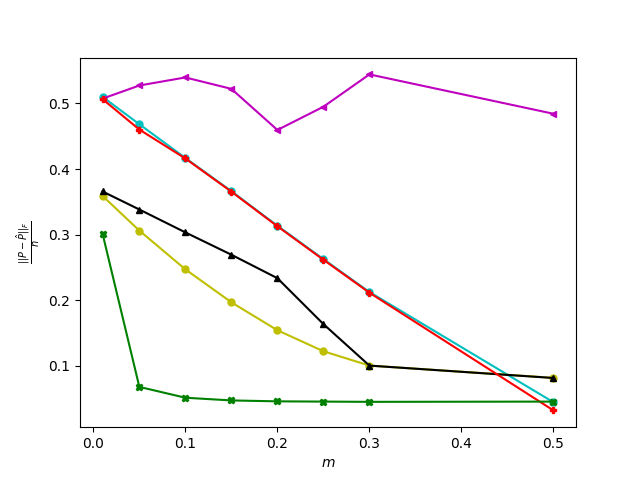} & \includegraphics[width = 0.3\textwidth]{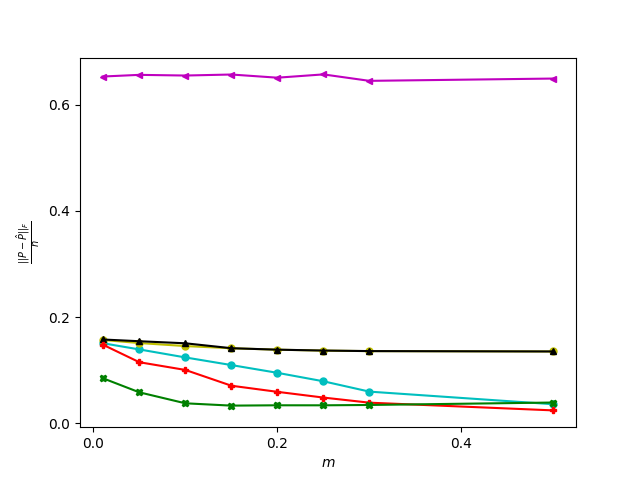} \\
   (d) & (e) & (f) \\
   \end{tabular}
   \includegraphics[width = 0.5\textwidth]{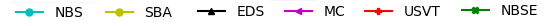}
\end{figure}

\begin{figure}[!htbp]
  \caption{Example $A^{\mathrm{obs}}$ (after suitable permutation of vertices to make the plots look nicer) in various types of missing data scenarios considered. Blacked out part of a maxtrix is not observed. (a) $T = 2$; (b) $T = 5$, for simulated graphons; (c) $T = 5$, with \textbf{frb59-26-4} data; (d) $T = 5$, with \textbf{bn-mouse-retina\_1} data; (e) $T = 5$, with \textbf{econ-beaflw} data.}
   \label{fig:covers}
   \centering
   \begin{tabular}{ccccc}
   \includegraphics[width = 0.18\textwidth]{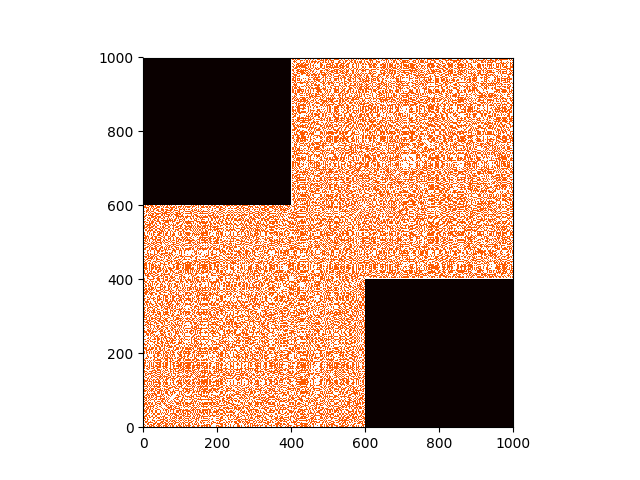} & \includegraphics[width = 0.18\textwidth]{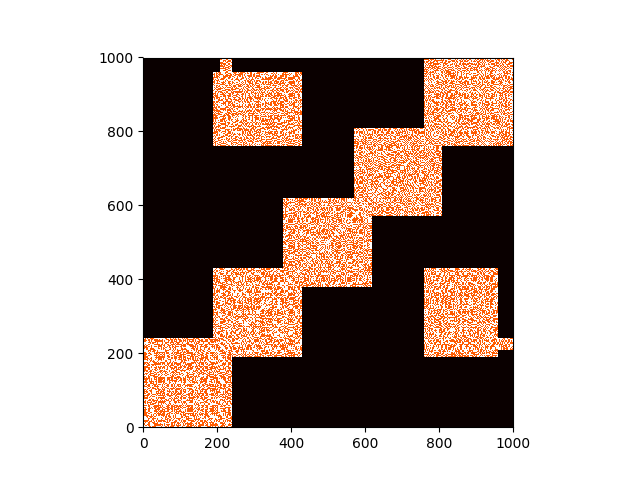} & \includegraphics[width = 0.18\textwidth]{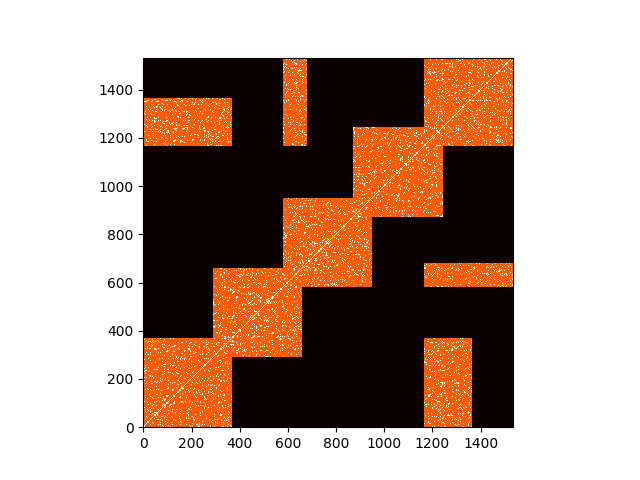} & \includegraphics[width = 0.18\textwidth]{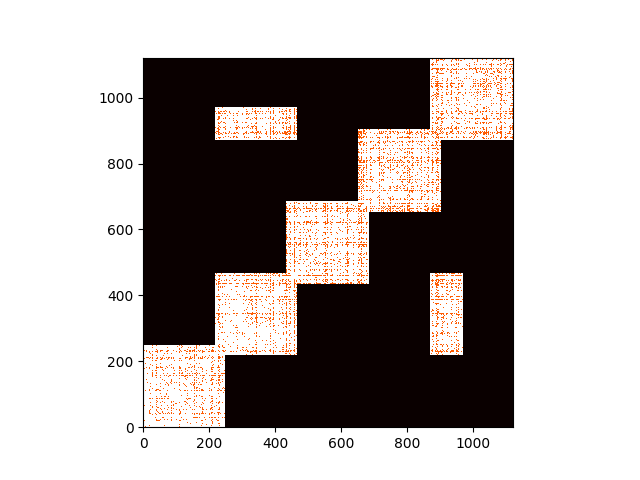} & \includegraphics[width = 0.18\textwidth]{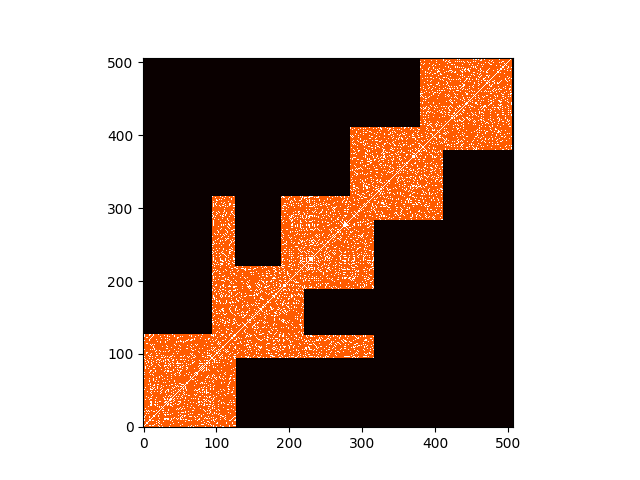} \\
   (a) & (b) & (c) & (d) & (e) \\
   \end{tabular}
\end{figure}

\begin{figure}[!htbp]
  \caption{Error of $\nbse$ when only one traversal (equivalent to paths here) traversal was used, versus when an averaged distance matrix over three paths was constructed as in the $\nbse_0$ algorithm (for graphons (a)-(f) respectively). The missing data scenario is as in Figure~\ref{fig:covers}-(b). Path 2, represented by the second bar, is in fact a maximal spanning tree of the corresponding super-graph.}
  \label{fig_barplot}
   \centering
   \begin{tabular}{ccc}
   \includegraphics[width = 0.30\textwidth]{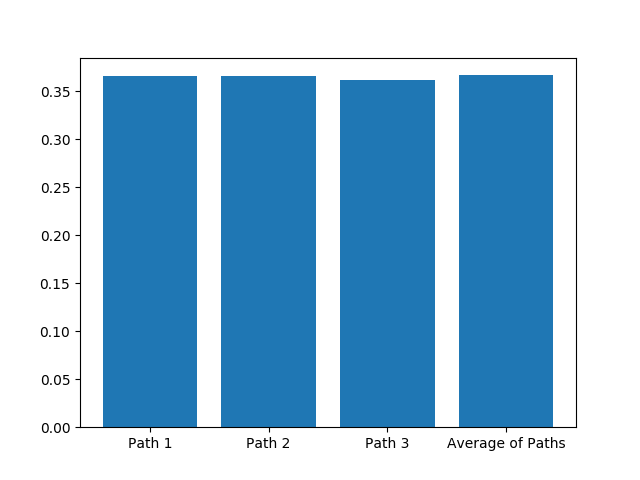} & \includegraphics[width = 0.30\textwidth]{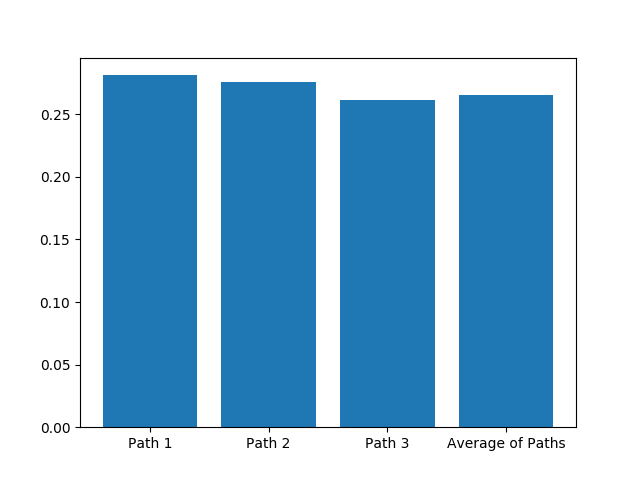} & \includegraphics[width = 0.30\textwidth]{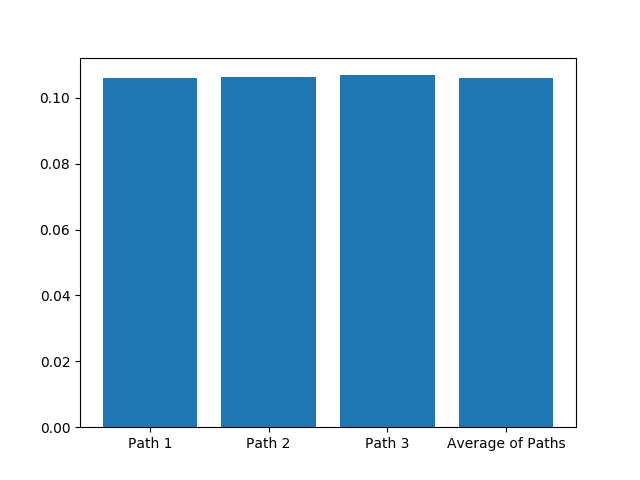} \\
   (a) & (b) & (c) \\
   \includegraphics[width = 0.30\textwidth]{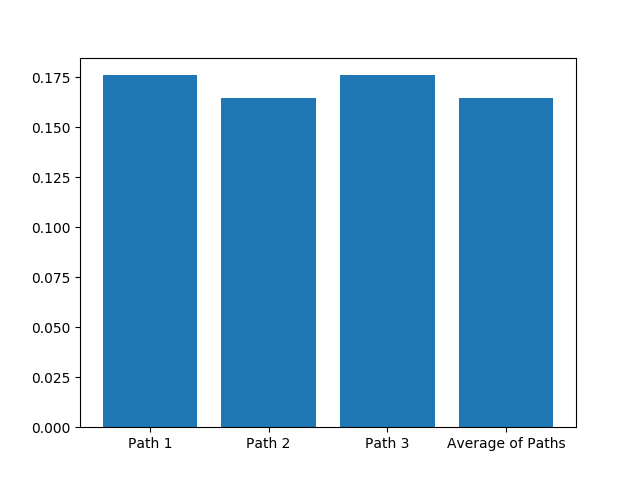} & \includegraphics[width = 0.30\textwidth]{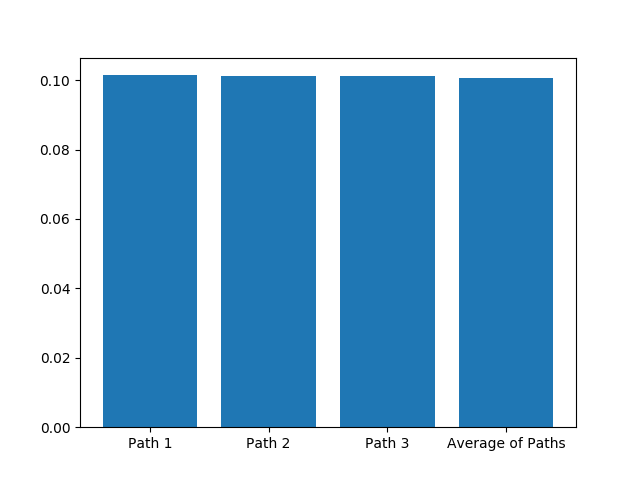} & \includegraphics[width = 0.30\textwidth]{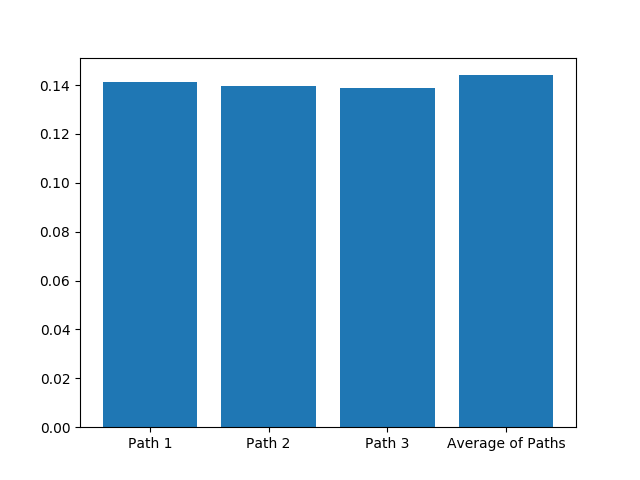} \\
   (d) & (e) & (f) \\
   \end{tabular}
\end{figure}

\begin{table}[!htbp]
\caption{Comparison of various graphon estimation techniques in a missing data scenario as depicted in Figure~\ref{fig:covers}-(b). For $\nbse$, we used the maximal spanning tree, which is a path in this example.}
\label{table1}
\centering
\begin{tabular}{|c||c|c|c|c|c|c|} \hline
     & NBSE & SBA & EDS & NBS & MC & USVT \\ \hline \hline
    (a) & \textbf{0.356634} & 0.479439  & 0.4978463 & 0.4884795 & 0.5099268   & 0.4646004 \\
    (b) & \textbf{0.132775} & 0.3179906 & 0.3132587 & 0.3244988 & 0.5186964 & 0.3057013 \\
    (c) &   0.100620  & 0.109924  & 0.1099362 & \textbf{0.09784396} & 0.6564036 & 0.1193309\\
    (d) & \textbf{0.163530} & 0.2055252 & 0.2047327 & 0.2102199 & 0.607879  & 0.1941453\\
    (e) & \textbf{0.100737} & 0.4689145 & 0.5058807 & 0.5782434 & 0.2328483 & 0.5729871\\
    (f) & \textbf{0.125929} & 0.1709241 & 0.1709545 & 0.1658264 & 0.5962249 & 0.1793617\\ \hline
 \end{tabular}
 \end{table}





\subsection{Real Data}
When applying the method on real networks, we do not actually have $P$. So, given a real network $A$, we first estimate $P$ based on the full graph, call it $\hat{P}_{\text{full}}$. Then we sample subgraphs $A_i, i = 1, \ldots, T$, with some degree of overlap, and based on these subgraphs only (i.e. on the incompletely observed full network) apply an algorithm to get an estimate $\hat{P}$. Then $\frac{1}{n}\|\hat{P} - P_{\text{full}}\|_F$ measures how much the incompleteness (or lack of overlap) influences the algorithm.

We do our experiments on three different datasets\footnote{All collected from Network Repository \citep{nr-aaai15}.} (see Figures~\ref{exp:realdata_gen} and \ref{exp:realdata_err}):
\begin{enumerate}
  \item \textbf{frb59-26-4:} This dataset contains benchmark graphs for testing several NP-hard graph algorithms including but not limited to the maximum clique, the maximum independent set, the minimum vertex cover and the vertex coloring problems. It has $1534$ nodes with about $1$ million edges.
  \item \textbf{bn-mouse-retina\_1:} In this dataset of a brain network edges represent fiber tracts that connect one vertex to another. It has $1122$ nodes and about $577.4$ thousand edges.
  \item \textbf{econ-beaflw:} This is an economic network that has $507$ nodes and about $53$ thousand edges.
\end{enumerate}

In Figure~\ref{exp:realdata_gen}, we plot the adjacency matrices of these graphs and also the probability matrix estimates obtained via $\nbse$ under various missing data scenarios. From Figure~\ref{exp:realdata_err}, we see that $\nbse$ suffers the least from lack of overlap.

\begin{figure}[!htbp]
  \caption{(a) Completely observed graph (orange indicates edges, white absence thereof); (b) estimated probability matrix from $\nbse$ when $m = 0.01n$, in the $T = 2$ case; (c) estimated probability matrix from $\nbse$ when $m = 0.1n$, in the $T = 2$ case; (d) estimated probability matrix from $\nbse$ when $m = 0.3n$, in the $T = 2$ case; (e) estimated probability matrix from $\nbse$ for a traversal in the general case ($T = 5$). Rows 1, 2, 3 correspond to the datasets \textbf{frb59-26-4}, \textbf{bn-mouse-retina\_1}, and \textbf{econ-beaflw} respectively.}
  \label{exp:realdata_gen}
  \centering
  \begin{tabular}{ccccc}
    \includegraphics[width=0.18\textwidth]{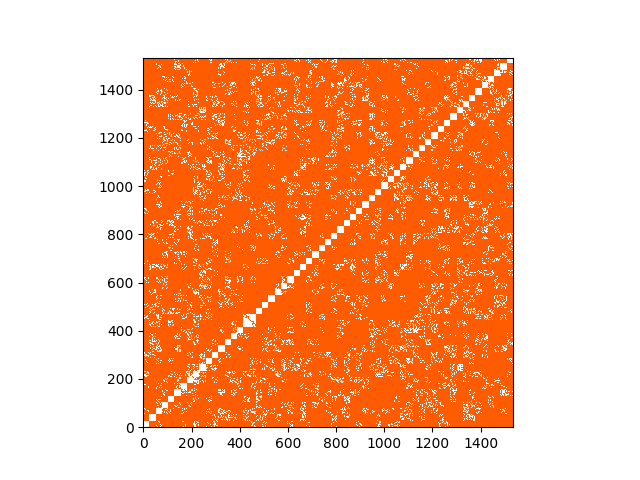} & \includegraphics[width=0.18\textwidth]{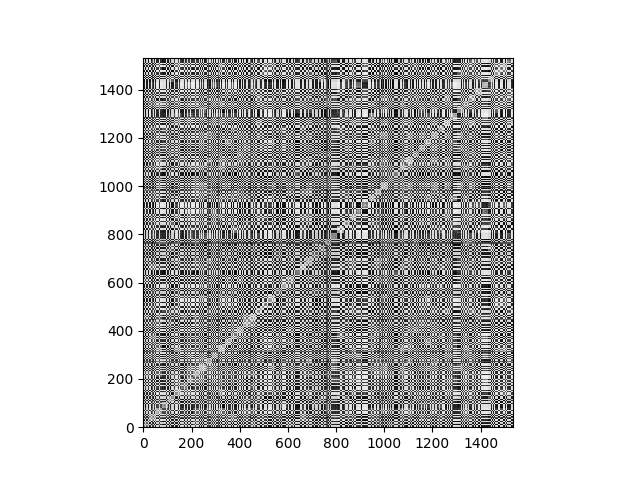} & \includegraphics[width=0.18\textwidth]{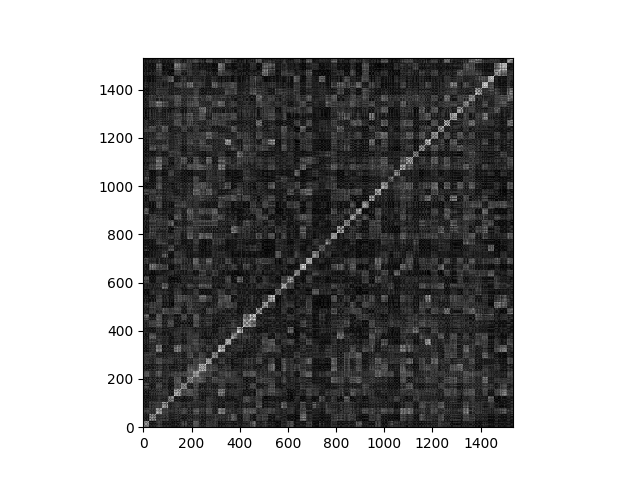} & \includegraphics[width=0.18\textwidth]{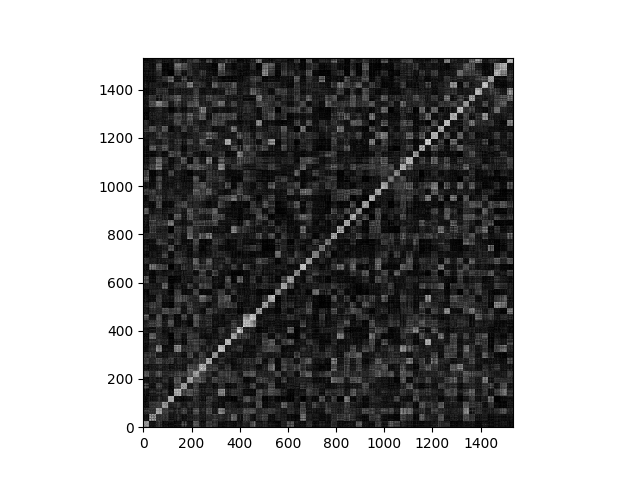} & \includegraphics[width=0.18\textwidth]{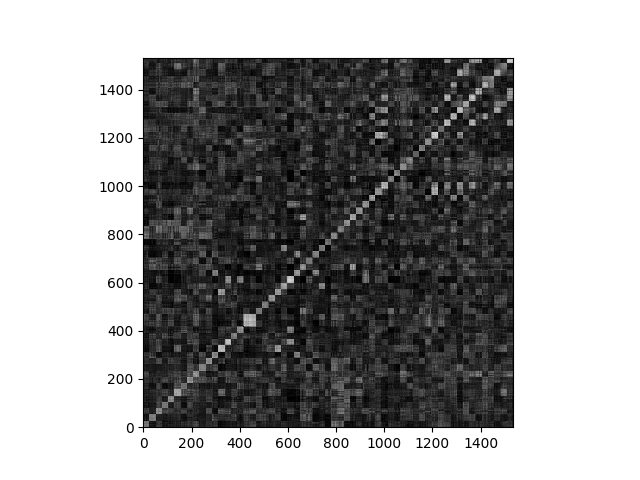} \\

    \includegraphics[width=0.18\textwidth]{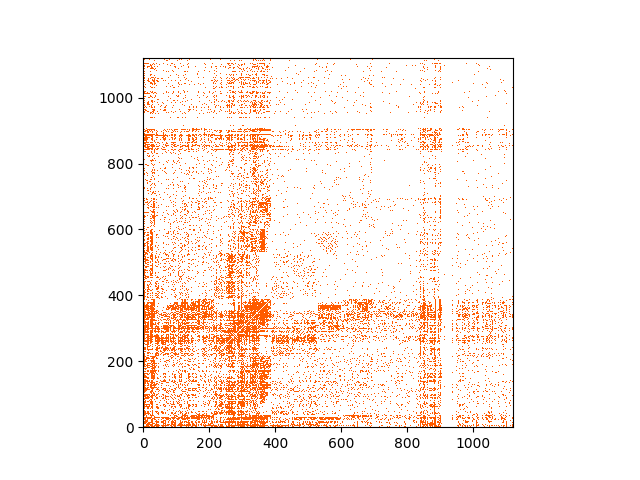} & \includegraphics[width=0.18\textwidth]{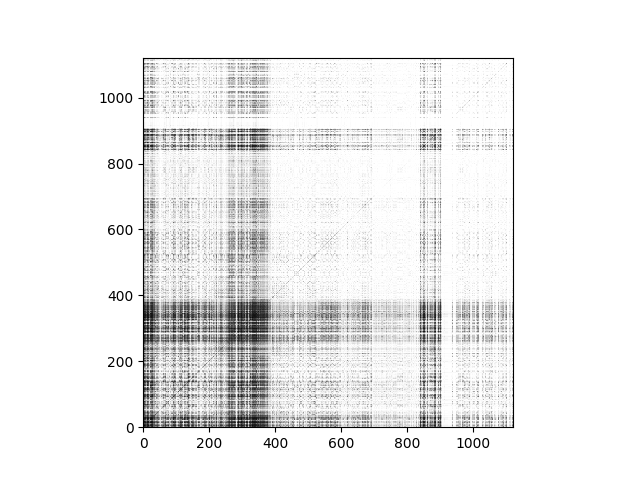} & \includegraphics[width=0.18\textwidth]{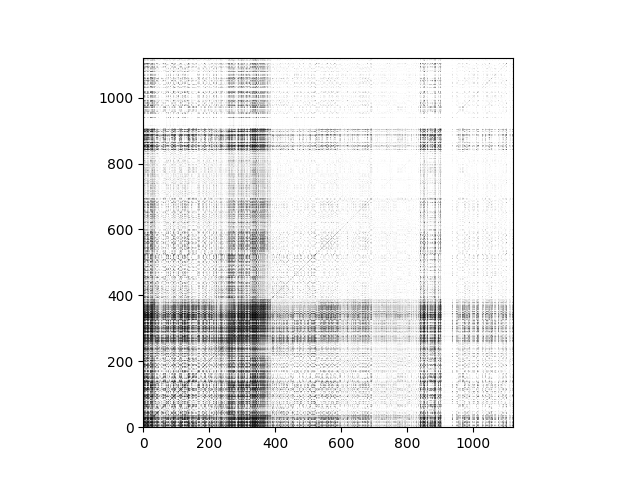} & \includegraphics[width=0.18\textwidth]{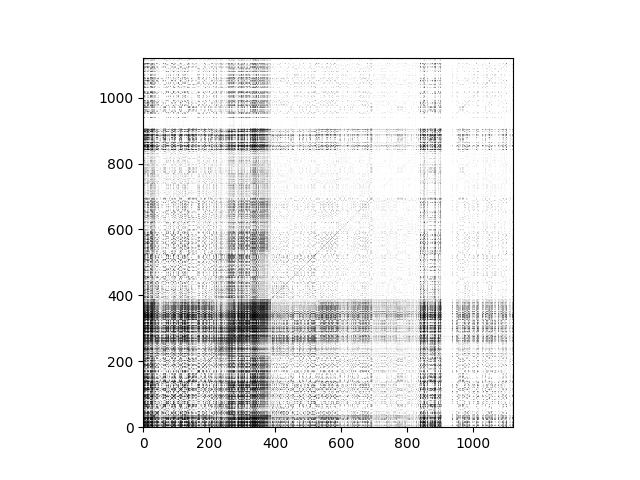} & \includegraphics[width=0.18\textwidth]{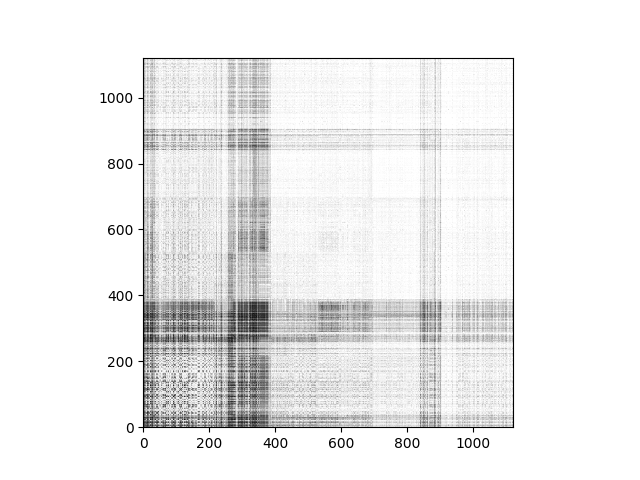} \\

    \includegraphics[width=0.18\textwidth]{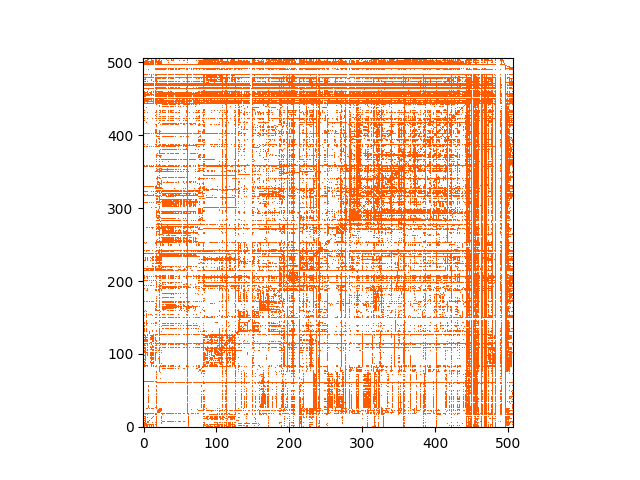} & \includegraphics[width=0.18\textwidth]{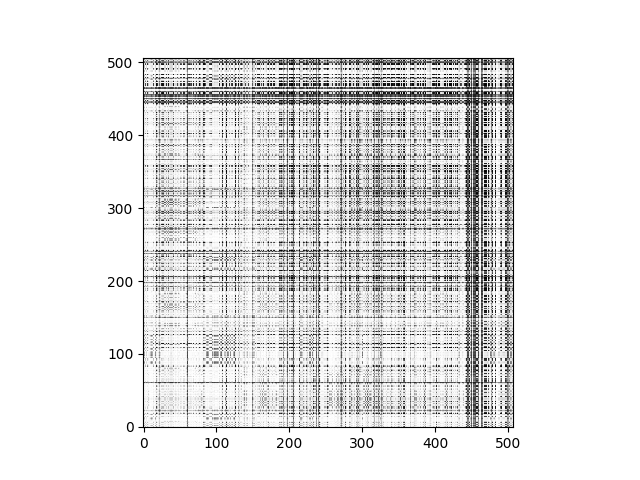} & \includegraphics[width=0.18\textwidth]{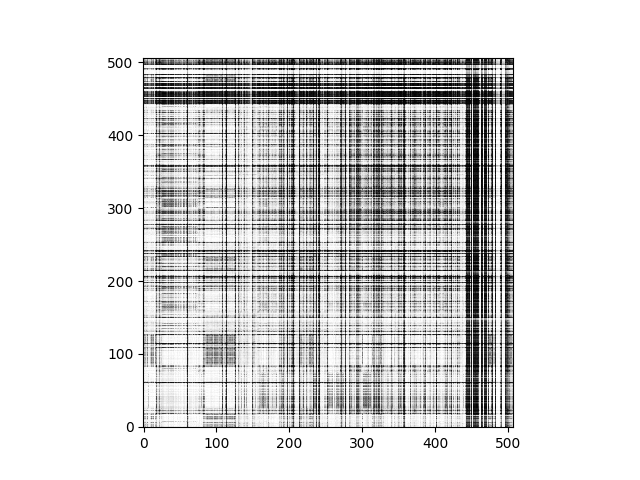} & \includegraphics[width=0.18\textwidth]{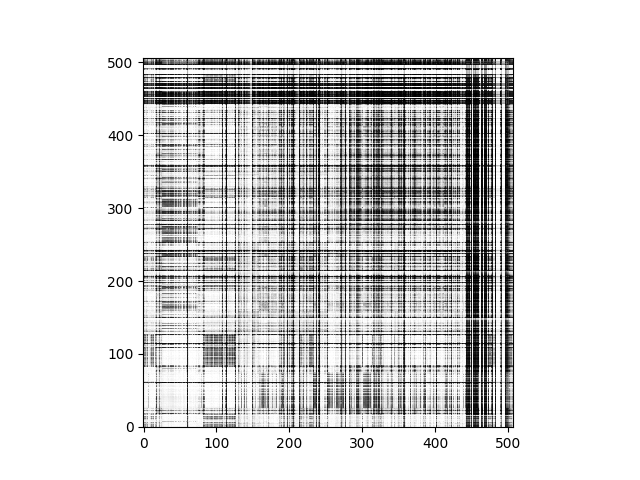} & \includegraphics[width=0.18\textwidth]{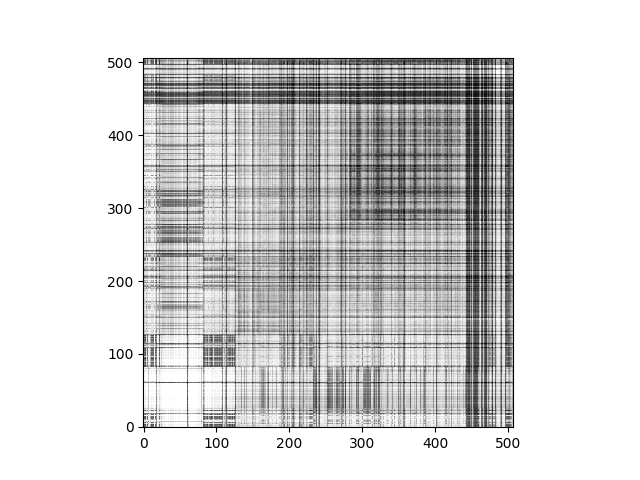} \\

    (a) & (b) & (c) & (d) & (e)
  \end{tabular}
\end{figure}

\begin{figure}[!htbp]
  \caption{Sensitivity of different graphon estimation algorithms to overlap $m$ in the scale of $n$ ($T = 2$).}
  \label{exp:realdata_err}
  \centering
  \includegraphics[width = 0.5\textwidth]{Figures/legend.jpg}
  \begin{tabular}{ccc}
    \includegraphics[width=0.27\textwidth]{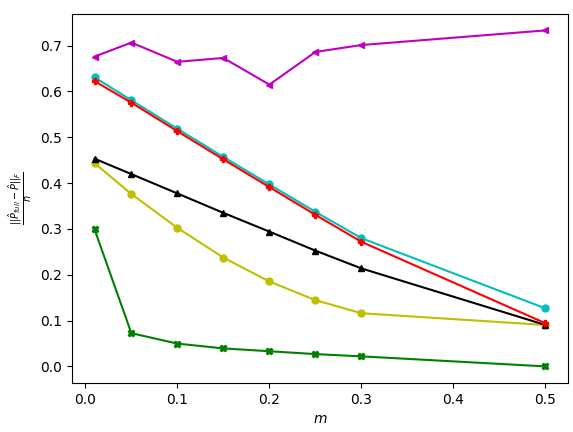} & \includegraphics[width=0.30\textwidth]{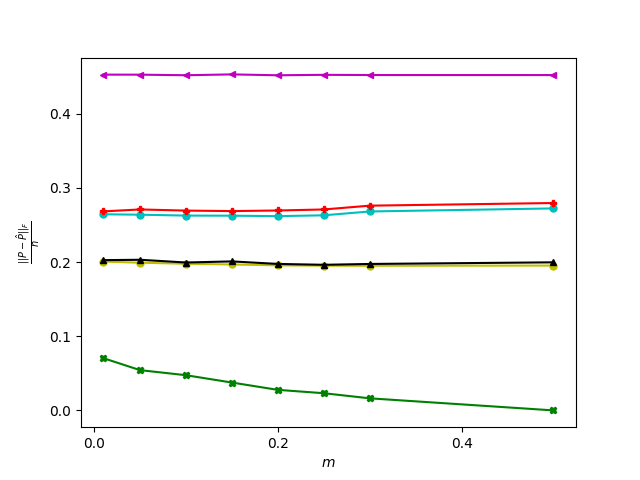} & \includegraphics[width=0.30\textwidth]{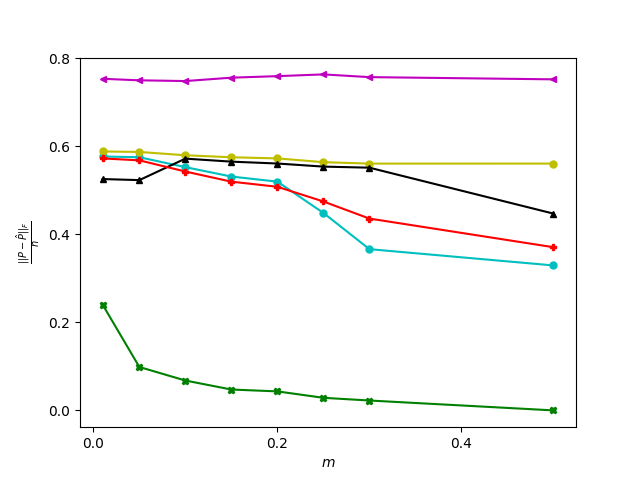} \\
    frb59-26-4 & bn-mouse-retina\_1 & econ-beaflw
  \end{tabular}
\end{figure}

\section{Conclusion}\label{sec:conc}
In conclusion, we have considered the estimation of the probability matrix of a network coming from a graphon model under a missing data set-up, where one only observes certain overlapping subgraphs of the network in question. We have extended the neighbourhood smoothing ($\nbs$) algorithm of \citet{zhang2017estimating} to this missing data set-up. We have shown experimentally that the proposed extension vastly outperforms standard graphon estimation techniques. We leave the study of theoretical properties such as obtaining the rate of convergence, how it depends on the degree of overlap, the number of subgraphs, etc. to future work.

\subsubsection*{Acknowledgments}
Thanks to Ananya Mukherjee for spotting an error in an earlier version of the paper.
\newpage
\bibliographystyle{apalike}
\bibliography{ige}

\end{document}